\documentclass{article}

\PassOptionsToPackage{numbers, compress}{natbib}

\usepackage[preprint]{neurips_2024}

\usepackage[utf8]{inputenc} 
\usepackage[T1]{fontenc}    
\usepackage{hyperref}       
\usepackage{url}            
\usepackage{booktabs}       
\usepackage{amsfonts}       
\usepackage{nicefrac}       
\usepackage{microtype}      
\usepackage[svgnames]{xcolor}         
\usepackage{subcaption}

\usepackage{amsmath,amssymb}
\usepackage{bm} 
\usepackage{algorithm}
\usepackage{algpseudocode}
\usepackage{graphicx}
\usepackage{textcomp}
\usepackage{amsthm}
\usepackage{multirow} 
\usepackage{framed}
\usepackage{color}
\usepackage{enumitem}

\newtheorem{theorem}{Theorem}[section]
\newtheorem{example}{Example}[section]

\newtheorem{lemma}[theorem]{Lemma}

\newtheorem{definition}[theorem]{Definition}

\title{Network Interdiction Goes Neural}

\author{%
  Lei Zhang\\
  Department of Computer Science\\
  Virginia Tech\\
  \texttt{zhanglei@vt.edu} \\
  \And
  Zhiqian Chen \\
  Department of Computer Science and Engineering \\
  Mississippi State University \\
  \texttt{zchen@cse.msstate.edu} \\
  \AND
  Chang-Tien Lu \\
  Department of Computer Science\\
  Virginia Tech\\
  \texttt{clu@vt.edu} \\
  \And
  Liang Zhao \\
  Department of Computer Science \\
  Emory University \\
  \texttt{liang.zhao@emory.edu} \\
}

\begin{document}

\maketitle

\begin{abstract}

Network interdiction problems are combinatorial optimization problems involving two players: one aims to solve an optimization problem on a network, while the other seeks to modify the network to thwart the first player's objectives.
Such problems typically emerge in an attacker-defender context, encompassing areas such as military operations, disease spread analysis, and communication network management.
The primary bottleneck in network interdiction arises from the high time complexity of using conventional exact solvers and the challenges associated with devising efficient heuristic solvers.
GNNs, recognized as a cutting-edge methodology, have shown significant effectiveness in addressing single-level CO problems on graphs, such as the traveling salesman problem, graph matching, and graph edit distance. Nevertheless, network interdiction presents a bi-level optimization challenge, which current GNNs find difficult to manage.
To address this gap, we represent network interdiction problems as Mixed-Integer Linear Programming (MILP) instances, then apply a multipartite GNN with sufficient representational capacity to learn these formulations. This approach ensures that our neural network is more compatible with the mathematical algorithms designed to solve network interdiction problems, resulting in improved generalization. Through two distinct tasks, we demonstrate that our proposed method outperforms theoretical baseline models and provides advantages over traditional exact solvers.
\end{abstract}

\section{Introduction}

In recent years, graph neural networks (GNNs) have exhibited promise in combinatorial optimization (CO) problems, with graphs being the preferred representation due to the discrete nature of most CO problems and the ubiquity of network data in real-world applications. 
For example, GNNs have also demonstrated effectiveness in CO problems such as Traveling Salesman Problem (TSP) \cite{prates2019learning}, graph matching \cite{Fey/etal/2020, li2019graph},  and graph edit distance \cite{bai2020learning}. 
GNNs' inductive bias makes them well-suited for encoding graph-structured data, benefiting from permutation invariance and sensitivity to input sparsity. 
Efforts have also been made to understand the limitations and underlying mechanisms of GNNs in CO problems \cite{bother2022s, Xu2020What, NEURIPS2022_8248b1de}, and to integrate GNNs with classical methods like beam search \cite{joshi2019efficient}  or tree search  \cite{li2018combinatorial} for improved performance.

Beyond the realm of CO problems where the optimization is unilateral, real-world scenarios often involve adversarial settings requiring \textit{two} levels of optimization. For instance, in a power grid system, it is not only important to design the network for maximum throughput (single level) but also to identify critical points where a compromise on the design of the network would cause the most significant damage (bi-level). This situation represents a more generalized CO problem with two competing roles: the \textit{follower} aims to solve a CO problem on a fixed network, while the \textit{leader} attempts to modify the network itself in a manner conflicting with the follower's objectives. This generic scenario is known as \textit{network interdiction}.

Traditional network interdiction solvers are commonly categorized as \textit{exact solvers} and \textit{heuristic solvers} \cite{smith2020survey}. Exact solvers aim to identify optimal solutions, necessitating the resolution of complex mathematical problems, typically NP-hard, which involves exploring a combinatorial space of potential network modifications and assessing their impact on the follower's solution \cite{cormican1995computational, wood1993deterministic}. As the network size increases, the exponential growth in the number of potential modifications results in computational complexity that exceeds polynomial time bounds. 
Conversely, heuristic solvers are algorithms designed to efficiently find good, if not optimal, solutions more quickly than exact methods. However, due to the complexity and variability of these problems—such as diverse assumptions, a vast solution space, and inherent problem intricacies—it's uncommon for a pure heuristic solver to be effective. Instead, combining heuristic modules with traditional algorithms offers greater promise \cite{smith2020survey}. GNN-based machine learning methods, which have shown effectiveness in conventional unilateral CO problems\cite{ijcai2021p595}, are clearly strong candidates for serving as heuristic modules in solving network interdiction problems more efficiently.

While it may seem intuitive to employ GNNs for addressing network interdiction problems, their application in this domain remains underexplored due to several challenges.
\textit{Firstly, there is no effective representation method for network interdiction instances.} For a GNN to learn to solve network interdiction problems, it must be able to encode all relevant information, including that of both the leader and the follower, in an attributed graph, and distinguish between instances with different solutions. 
\textit{Secondly, there is no theoretical guarantee that GNNs can be trained to perform algorithmically and generalize in solving network interdiction problems.}
Recent studies have shown that GNNs excel in certain CO tasks because they ``align'' with dynamic programming (DP) \cite{Xu2020What}. Given that DP or analogous polynomial-time heuristics can solve numerous CO problems, GNNs offer potential for generalization and extrapolation in these areas \cite{NEURIPS2022_8248b1de}. Unfortunately, network interdiction problems do not align with DP, making previous GNN for CO unable to be trivially used for network interdiction.

This paper proposes a new computational framework to solve the network interdiction problem. To do so, we answer two key questions: 1) \textit{Can GNNs provide effective representation for network interdiction problems, and if so, what theoretical assurances exist?} 2) \textit{If the alignment between DP and GNN does not justify the utilization of GNNs for network interdiction problems, what alternative rationale does? }
To answer these two questions, we leverage existing research on the expressive capabilities of GNNs while accounting for the unique attributes of network interdiction problems and their conventional solution methodologies. Specifically, we employ a multipartite graph structure to comprehensively represent network interdiction problems and develop a novel GNN model tailored for multipartite graphs, termed MMILP-GNN. This proposed framework demonstrates that the MMILP-GNN model, supplemented with the random feature trick, offers sufficient representation power for network interdiction problems. Furthermore, by design, the framework aligns with traditional algorithms used to solve network interdiction problems, thereby not only fitting the training data but also demonstrating a degree of generalization ability.

\section{Preliminaries}

A network interdiction problem typically includes two players engaging in a game of max-min or min-max on a defined graph. One player, commonly referred to as the follower or defender, aims to optimize a standard CO problem on the graph, such as identifying the shortest path or maximizing the maximum flow between two nodes. The other player, often known as the leader or attacker, manipulates the network on which the follower operates, strategically disrupting the follower's objective by actions like removing edges that should be connected in the graph or adding costs to existing edges. Formally, we define an instance of a network interdiction problem as follows. 

\begin{definition} [Network Interdiction Problem]\label{def:ni}
The general form of a max-min network interdiction problem is defined as:
\begin{equation}
    \max \Theta(\mathbf{x})  \quad \quad  s.t. \quad \mathbf{x} \in X,
\end{equation}
where $\mathbf{x}$ represents the leader/attacker's variable. The objective function of the interdiction $\Theta(\mathbf{x})$ is defined as the minimization of another function: 
\begin{equation}
    \Theta(\mathbf{x})  = \min f(\mathbf{x}, \mathbf{y}) \quad \quad   s.t. \quad \mathbf{y} \in Y(\mathbf{x}),
\end{equation}
where $\mathbf{y}$ represents the follower/defense variables,  $f(\mathbf{x}, \mathbf{y})$ represents the follower/defender's objective function (affected by the leader/attacker's action $\mathbf{x}$), and $Y(\mathbf{x})$ is the set of feasible actions that the follower/defender can do for a given $\mathbf{x}$. 

A min-max network interdiction problem is essentially the reverse of the max-min network interdiction problem, where the outer objective is now minimization, and the inner objective is maximization.
\end{definition}

\begin{example} [Shortest Path Interdiction]\label{exp:test}

\begin{equation}\small
\begin{split}
    \max_{\mathbf{x} \in X}  \{ \min_{\mathbf{y}\in Y(\mathbf{x})} \sum_{(i,j)\in A} (c_{i,j}+d_{i,j}x_{i,j})y_{i,j} \} \quad 
    \textrm{s.t.} \quad   x \in \{ 0, 1\}^{|A|} : \sum_{(i,j)\in A} x_{i,j}\leq \gamma,  \quad T\textbf{y} = \textbf{b} ,\quad  y\geq 0 ,
\end{split} \label{exp:1}
\end{equation}

\end{example}

Eq. (\ref{exp:1}) illustrates the mathematical formulation of the shortest-path network interdiction problem. In this context,  $\mathbf{x} = \{ x_{i,j}\}_{(i,j)\in E}$ is a binary decision vector indicating interdicted edges, with $E$ being the network's edge set. Here, $c_{i,j}$ denotes the length of the edge, $d_{i,j}$ represents the additional length if edge $(i,j)$ is interdicted, $y(i,j)$ indicates whether the edge exists in the network,  and $ \gamma$ is the budget constraint on the number of interdictions. 

\textbf{Traditional exact solvers}: Branch-and-Bound (BB) and  Benders decomposition are two fundamental traditional exact methods for solving network interdiction problems.  Compared to Benders decomposition, BB is more versatile and applicable to a wide range of network interdiction problems \cite{du2013handbook}. The problems that BB can handle are normally defined in MILPs, which can be formulated as 
\begin{equation} \small 
    \min_{x \in \mathbb{R}^n} c^Tx, \quad \textrm{s.t.}; Ax \circ b, l \leq x \leq u, x_j \in \mathbb{Z}, \forall j \in I .\label{eq:milp}
\end{equation}
For general MILPs, the worst-case complexity can be similarly exponential, as the algorithm may need to consider an exponential number of subproblems in the search space. Efforts have been made to represent general LPs and MILPs using a bipartite graph structure and then utilize GNNs to estimate solutions \cite{chen2022representing, chen2022representing2,gasse2019exact}, but these methods have not yet been applied to network interdiction problems.

\section{Proposed Framework}

This section outlines the process of transforming a network interdiction problem into input suitable for a neural network, along with elucidating the neural network's architecture. 
To ensure the model's generalization ability, in Section \ref{sec:rep}, we process the problem instances into the Mixed Integer Linear Programming (MILP) form that BB can handle, following traditional methods. Given that GNNs have shown effectiveness in learning branching strategies \cite{ijcai2021p595}, this approach simplifies the reasoning task for the GNN.
These MILP formulations are then translated into multipartite graphs, capturing essential characteristics of the scenarios as outlined in Section \ref{sec:mmilp}.
Subsequently, we propose a multipartite graph neural network, MMILP-GNN, tailored for estimating optimal interdiction decisions on the induced multipartite graphs rather than the original competitive network between the leader and follower, as elucidated in Section \ref{sec:mmilpgnn}. 
More detailed theoretical rationale is provided in Section \ref{sec:theory}.

\subsection{Preprocess Network Interdiction Instances} \label{sec:rep}

As the initial step in solving network interdiction problems, our aim is to process the problem instances using traditional methods up to the point where these methods become inadequate, leaving the challenging part to GNNs. As illustrated in Example \ref{exp:test}, we have already transformed the problem instances into a constrained optimization problem. However, this format is not suitable for neural networks due to the nested two levels of optimization. To simplify this, we employ the ``dualize-and-combine'' approach. First, we derive the dual formulation of the inner problem with a fixed leader decision variable (in Example \ref{exp:test}'s case, the variable $x$), ensuring that both players' problems align in the same optimization direction. Subsequently, we release the leader decision variable as a decision vector, converting the whole problem into a single-level MILP. 

It is noteworthy that the network interdiction problems addressed in this paper, represented by Eq. (\ref{exp:1}), involve a followers' problem that lends itself to being modeled as a convex optimization problem. This forms the basis for applying the ``dualize-and-combine" technique for single-level reduction. Although this imposes a constraint on the types of problems we can tackle, the majority of network interdiction problems commonly encountered in real-world scenarios fall within this category \cite{smith2020survey}.

Take the shortest path interdiction problem in Eq. (\ref{exp:1}) as an example, following the above-mentioned processes, the constrained bi-level optimization problem in Eq. (\ref{exp:1}) can be transformed into a single-level optimization: 
\begin{example}[Single-Level Reduction for the Shortest Path Interdiction Problem]
\begin{equation} \small
\begin{split}
    \max_{\mathbf{x} \in X}  \{ \min_{\mathbf{y}\in Y(\mathbf{x})} \sum_{(i,j)\in A} (c_{i,j}+d_{i,j}x_{i,j})y_{i,j} \}  \quad &
    \textrm{s.t.} \quad   x \in \{ 0, 1\}^{|A|} : \sum_{(i,j)\in A} x_{i,j}\leq \gamma, \quad T\textbf{y} = \textbf{b} ,\quad  y\geq 0 , \\
    \quad  \quad \quad  & \Downarrow \\
    \max_{\textbf{x}, \bm{\pi}}  \quad  \textbf{b}^T \bm{\pi} \quad &
     \textrm{s.t.} \quad   T^T \bm{\pi} \leq \textbf{c} + D\textbf{x} , 
   \quad  \textbf{x} \in \{ 0, 1\}^{|E|} : \sum_{(i,j)\in A} x_{i,j}\leq \gamma
\end{split} \label{eq:reduction}
\end{equation}
\end{example}

\subsection{Graph Representations for Network Interdiction Problems} \label{sec:mmilp}

We propose a \textit{Multipartite MILP-induced graph} or \textit{MMILP-graph} representation to encode the above single-level optimization into a form, i.e., MMILP-Graph, that is readable by GNNs while preserving all the needed information.

\begin{definition} [MMILP-graph] A multipartite MILP-induced graph is defined as a tuple $(G, H)$, where $G\equiv (W_0 \cup W_1 \cup ... W_p \cup V, E)$ consists of vertex set $W_0 \cup W_1 \cup ... W_p \cup V$ and a weighted edge set $E$, and $H$ represents vertex features.
\begin{itemize}[leftmargin=10pt]
\item \textbf{Vertices}: The vertex set can be divided into three groups including one interdiction action variable vertex group $W_0$, $m$ dual-variable vertex group, and one constraint vertex group $V$. Each of the child vertex groups has no overlap with another child vertex group. For convenience, we use $\mathit{V'}$ to represent one of the vertex groups: $\mathit{V'} = (W_0, W_1 , ... ,W_p, V)$. 
\item \textbf{Edges}: One edge in $E$ can connect one variable vertex in any of the variable vertex groups with one constraint vertex. Note that there is no edge connecting vertices in the same vertex group, or between two different vertex groups. An edge can carry edge features: $E = E^{W_0, V} \cup E^{W_1, V} \cup , ..., \cup E^{W_p, V}$ where $E^{W_k, V}$ represents the edges connecting one variable in $W_p$ and one constraint in $V$. 
\item \textbf{Vertex features}: Each vertex has its corresponding feature vector that represents characteristics in the original interdiction problem. 
\end{itemize}
Finally, an MMILP-graph is defined as $(G, H) \in \mathcal{G}$
\label{def:mmilp}
\end{definition}

\begin{example}[MMILP-Graph for a Specific Shortest Path Interdiction Instance]
Take the shortest path interdiction instance in Fig. \ref{fig:example} as an example, the induced MILP can be found in Eq. (\ref{eq:reduced_example}), and the corresponding MMILP-graph is shown in Fig. \ref{fig:mmilp_example}. 
\begin{equation} \small 
\begin{split}
& \max  (\pi_0 - \pi_9) \\ 
\textrm{s.t.} \quad & v_{1}: \pi_0-\pi_1 - x_{0,1} \leq 9, 
    v_{2}: \pi_0-\pi_4 - x_{0,4} \leq 3,  
    v_{3}: \pi_0-\pi_3 - x_{0,3} \leq 3, \\
    & v_{4}: \pi_1-\pi_2 - x_{1,2} \leq 4, 
    v_{5}: \pi_1-\pi_3 - x_{1,3} \leq 1, 
    v_{6}: \pi_4-\pi_3 - x_{4,3} \leq 2,  \\
    & v_{7}: \pi_4-\pi_5 - x_{4,5} \leq 3, 
    v_{8}:\pi_3-\pi_2 - x_{3,2} \leq 8, 
    v_{9}: \pi_3-\pi_6 - x_{3,6} \leq 4, \\
    & v_{10}: \pi_3-\pi_5 - x_{3,5} \leq 6, 
    v_{11}: \pi_2-\pi_6 - x_{2,6} \leq 5, 
    v_{12}: \pi_5-\pi_6 - x_{5,6} \leq 4, \\ 
 & v_{13}: x_{0,1} + x_{0,3} + x_{0,4} + x_{1,3}  + x_{1,2} + x_{2,6}   + x_{3,2} + x_{3,6} + x_{3,5} + x_{4,3} + x_{4,5} + x_{5,6} \leq 1  \\
\end{split}\label{eq:reduced_example}
\end{equation}

\begin{figure}
\centering
\begin{subfigure}{.5\textwidth}
  \centering
  \includegraphics[width=.8\linewidth]{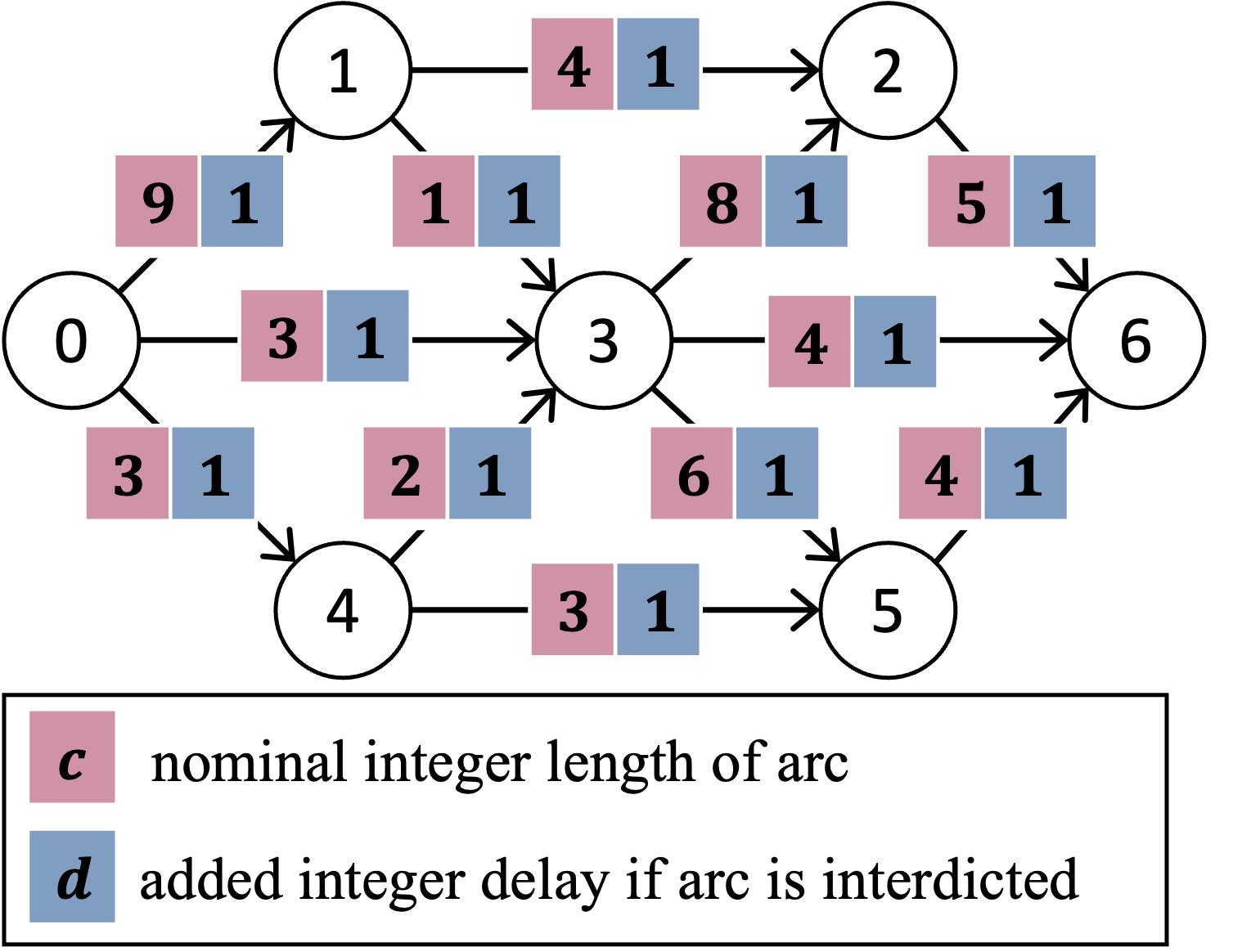}
  \caption{Example Shortest Path Interdiction Instance}
  \label{fig:example}
\end{subfigure}%
\begin{subfigure}{.5\textwidth}
  \centering
  \includegraphics[width=.7\linewidth]{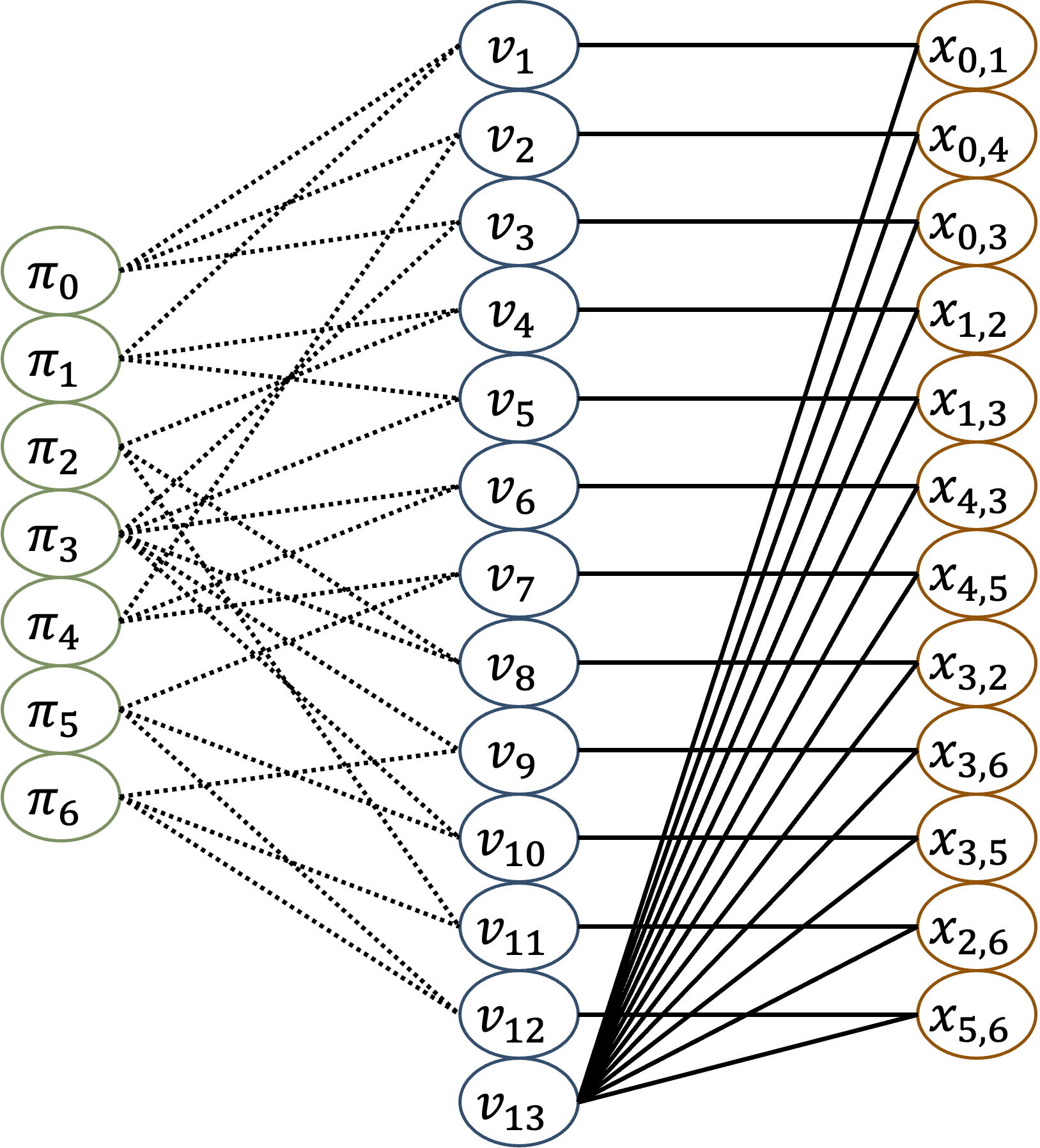}
  \caption{MMILP-graph for Eq. (\ref{eq:reduced_example})}
  \label{fig:mmilp_example}
\end{subfigure}
\caption{Network Interdiction Instance and the Corresponding MMILP-Graph}
\label{fig:example_combined}
\end{figure}
\textbf{Vertices}: In in Fig. \ref{fig:mmilp_example}, the vertices groups include one interdiction action variable group $\mathbf{x} = \{x_{0,1}, x_{0,3},...,x_{5,6}\}$ (the vertices on the right side ), only one dual-variable vertex group $\boldsymbol{\pi} = \{\pi_0, \pi_1, ..., \pi_6\}$ (the vertices on the left side), and one constraint group $\mathbf{w} = \{w_1, w_2, ..., w_{13}\}$ (the vertices in the middle). In this problem, $p$ in Definition \ref{def:mmilp} equals $1$ because there is only one dual variable $\boldsymbol{\pi}$ in the reduced form of MILP. 
    
\textbf{Edges}: One edge in Fig. \ref{fig:mmilp_example} connects one constraint vertex (in the middle column) and one variable vertex (in the left or right column), representing the relationship between a constraint ($v_i, i=1,2,...,13$) and the involved variables in Eq (\ref{eq:reduced_example}). The weight of the edge represents the coefficient of the variable in the constraint. For example, the weight of the edge between vertex $\pi_0$ and $v_1$, i.e., $E^{W_1, V}_{0,1}$, is 1 because in the constraint $v_1$, i.e., $\pi_0-\pi_1-x_{0,1} \leq 9$, the coefficient of $\pi_0$ is 1. 

\textbf{Vertex features}: The vertex features for both variable vertices and constraint vertices align with the feature definitions in the bipartite graphs in \cite{gasse2019exact, nair2020solving, han2022gnn, chen2022representing, chen2022representing2}. $h_i^{\mathit{V}'}$ represents the feature vector for the $i$-th node of vertex group $V'$.
\end{example}

\textbf{Relationship with existing work}.  This MMILP-graph representation is inspired by the use of bipartite graphs for solving general MILPs, as initially introduced by \citet{gasse2019exact} and subsequently utilized in various studies ( \cite{nair2020solving}, \cite{han2022gnn}, \cite{chen2022representing}, \cite{chen2022representing2}). 
Although we can demonstrate that the induced single-level MILPs of network interdiction instances in Section \ref{sec:rep} can be represented in the standard MILP format, the use of MMILP-graph can not only capture the heterogeneity of different variables but also learn their different degrees of importance. 
Firstly, variables in the induced problems typically are different from each other. Comparing Eq. (\ref{eq:reduction}) with a general MILP in Eq. (\ref{eq:milp}), the dual vector $\pi$ and the interdiction decision vector $x$ represent distinct variables. Incorporating them into the MILP-graph overlooks the difference between $x$ and $\pi$, resulting in reduced information. 
Secondly, different variables have varying degrees of importance. Take shortest-path interdiction as an example, our primary interest lies in the value assignments of variable $x$ because once $x$ is assigned, the problem simplifies to a regular shortest-path problem that can be solved by Dijkstra's algorithm within polynomial time. This characteristic is also called decomposability which is essential for using Benders Decomposition on network interdiction problems.

\subsection{Graph Neural Network for MMILP-Graphs} \label{sec:mmilpgnn}
Given the above-formulated MMILP-graph, here we propose a Multipartite Graph Neural Network that takes it as input to predict the solution network interdiction instance. 
The graph neural network aims to learn an $\mathcal{G} \rightarrow \mathbb{R}^n$ mapping where $n$ is the number of nodes in the network interdiction problem as well as the number of variables in vertex group $W_0$. 

In detail, because of the heterogeneity and multipartite structure of the input graph, our graph neural network has several passes. In general, we design the encoding layer for each partite of the vertices; we then propose a new message-passing strategy to learn the mapping between different groups of vertices where they should communicate; and finally, a read-out layer will produce estimates for the interdiction variable vertices, each corresponding directly to one candidate edge in the original network.

\textbf{Variable and constraint vertices encoding}: Raw vertex features are encoded into the embedding space with trainable functions $f_{in}^{\mathit{V}'}, (\mathit{V}' = (W_0, W_1, ..., W_p, V)$): 
\begin{align}
\begin{split}
h_i^{0,\mathit{V}'} = f_{in}^{\mathit{V}'}(h_i^{\mathit{V}'}), 
\end{split}\label{eq:input_layer}
\end{align}
where $h_i^{0,\mathit{V}'}$ represents the embedding (layer 0) for node $i$ in vertex group $\mathit{V}'$, and $f_{in}^{\mathit{V}'}$ represents the function for encoding the raw vertex features in group $\mathit{V}'$ into embedding space.  In practice, we append random features \cite{sato2021random} to the vertex features to enhance the separation power of our method as explained in Section \ref{sec:theory}. 

\textbf{Variables-constraints message passing}: We stack multiple graph neural network layers with variables-to-constraints message passing ($v\rightarrow c$) and constraint-to-variable message passing ($c\rightarrow v$). Since there are multiple sets/groups of variables in the MMILP-graph, the trainable functions are different as well. $v\rightarrow c$ and $c\rightarrow v$ are defined in turn in Eq (\ref{eq:v2c}) and Eq (\ref{eq:c2v}): 
\begin{equation} \small 
\begin{split}
& h_i^{l, V} = g_{l}^{V}(
                h_i^{l-1,V}, 
                \sum_{k=0}^{p} \sum_{e\in E^{W_k, V}}^{} f_l^{W_k}(h_i^{l-1,V}, h_j^{l-1,W_k}, e ),
\end{split}\label{eq:v2c}
\end{equation} 
\begin{equation} \small 
\begin{split}
& h_i^{l, W_k} = g_{l}^{W_k}(
                h_i^{l-1,W_k}, 
                 \sum_{e\in E^{W_k, V}}^{} f_l^{W_k}(h_i^{l-1,V}, h_j^{l-1,W_k}, e ), 
\end{split}\label{eq:c2v}
\end{equation}
where $h_i^{l, W_k}$ is vertex features for node $i$ within variable set $W_k$ at layer $l$, and $h_i^{l, V}$ is vertex features  for node $i$ in constraint set $V$ at layer $l$. 

\textbf{Interdiction decision read-out}: The read-out function is only applied on the interdiction decision variable vertices, i.e., $W_0$ as is demonstrated in Eq (\ref{eq:output_layer}). 
\begin{equation}
\begin{split}
& \hat{x_i} = g_{out}( h_i^{L,W_0}), 
\end{split}\label{eq:output_layer}
\end{equation}

With the definition of the detailed operations, we denote the collection of GNNs with $\mathcal{F}$: 
\begin{equation}
\begin{split}
\mathcal{F} = \{ F: \mathcal{G} \rightarrow \mathbb{R}^n | F \; \textrm{yields (\ref{eq:input_layer}), (\ref{eq:v2c}), (\ref{eq:c2v}), (\ref{eq:output_layer})} \}
\end{split}\label{eq:gnn}
\end{equation}

In practice, we use multi-linear perceptrons (MLPs) for all the trainable functions in Eq (\ref{eq:input_layer}) - (\ref{eq:output_layer}). 

Since we train the models using optimal solutions as labels, the predicted values can be interpreted as the likelihood of whether an edge should be interdicted or the budget to be allocated to that edge. This is actually more convenient than most unilateral CO problems that involve sequential decision-making. For instance, solutions to the shortest path problem must consist of a sequence of connected edges in the graph from the source to the target, whereas shortest path network interdiction does not have this requirement.

\section{Theoretical Analysis} \label{sec:theory}

\subsection{Representation Power Analysis}

The representation power of a neural network is determined by its ability to produce different outcomes for different inputs, which is fundamental to its effectiveness. Our approach to measuring the representation power of MMILP-GNN is heavily inspired by existing research on the representation power of GNNs \cite{Xu2020What} and the separation power of GNNs in LP \cite{chen2022representing} and MILP \cite{chen2022representing2}. Following the extension of the Weisfeiler-Lehman (WL) test for LP-Graphs \cite{chen2022representing} and MILP-Graphs \cite{chen2022representing2}, we further extend it for MMILP-Graphs as shown in Algorithm \ref{alg1}.
We extend these results to MMILP-GNNs for network interdiction instances and validate them in a similar manner.

\begin{algorithm*}[htb!]\scriptsize
	\caption[WL test for MMILP-Graphs]{ $\mathrm{WL}_{\mathrm{MMILP}}$}\label{alg:WL}
	\begin{algorithmic}[1]
		\Require A graph instance $(G,H)$ with $p$ dual variables, iteration limit $L>0$.
		\State  $C_i^{0,V} = \text{HASH}_{V}^{0}(h_i^V)$
            \State $C_j^{0,W_k} = \text{HASH}_{W_k}^{0}(h_j^{W_k})$ for $k \; \textrm{in} \; \{0,1,2,..., p\}$.
		\For{$l=1,2,\cdots,L$}
		\State $C_i^{l,V} = \text{HASH}_{V}^{l} \left(C_i^{l-1,V}, \sum_{k=0}^{p}\sum_{j=1}^{n_k} E_{i,j}^{W_k, V}\text{HASH}_{W_k}^{l'}\left(C_j^{l-1, W_k}\right)\right)$.
            \For{$k=0, 1,\cdots,p$}
		\State $C_j^{l,W_k} = \text{HASH}_{W_k}^{l} \left(C_j^{l-1,W_k}, \sum_{i=1}^m E_{i,j}^{W_k,V} \text{HASH}_{V,W_k}^{l'}\left(C_i^{l-1, V} \right)\right)$.
            \EndFor
		\EndFor
		\State \textbf{return} The multisets 
		containing all colors $\{\{ C_i^{L,V} \}\}_{i=0}^m$ and $ \{ \{\{ C_j^{L,W_k} \}\}_{j=0}^n \}_{k=0}^p$.
	\end{algorithmic} \label{alg1}
\end{algorithm*}

\begin{theorem}
For two MMILP-graphs $(G, H)$ and $(\hat{G}, \hat{H})$, they cannot be distinguished by $\mathrm{WL}_{\mathrm{MMILP}}$ if and only if $F(G,H) = F(\hat{G}, \hat{H}), \forall F\in \mathcal{F}$\label{theom:main}
\end{theorem}

The proof of Theorem \ref{theom:main} follows from Lemma \ref{lem:sameWL2sameWGNN} and \ref{lem:sameGNN2sameWL}.

\begin{lemma}
Let $(G,H), (\hat{G}, \hat{H}) \in \mathcal{G}$. if $(G,H) \sim (\hat{G}, \hat{H})$  then for any $F_W\in\mathcal{F}_{\text{GNN}}^W$, there exists a permutation $\sigma_W\in S_n$ such that $F_W(G,H) = \sigma_W( F_W(\hat{G},\hat{H}))$. \label{lem:sameWL2sameWGNN}
\end{lemma}

\begin{lemma}\label{lem:sameGNN2sameWL}
	Let $(G, H),(\hat{G},\hat{H})\in \mathcal{G}$. If $F(G,H) = F(\hat{G},\hat{H})$ holds for any $F\in \mathcal{F}$, then $(G,H)\sim (\hat{G},\hat{H})$. 
\end{lemma}

The proofs are detailed in Appendix \ref{appendix:proof}. Since our results align with those in \cite{chen2022representing2} for MILPs, it is possible that the network interdiction instances induce \textit{foldable} MILPs as introduced in their work. To address these cases, we enhance the representation power in our experiments by appending random features \cite{sato2021random} to the MMILP-graphs.

\subsection{Discussions on Algorithmic Alignment }

The concept of algorithmic alignment, introduced by \citet{Xu2020What}, describes the reasoning capabilities of GNNs. The core idea is that a neural network is more effective at learning to perform a reasoning task, such as solving a CO problem, when its architecture aligns well with the algorithm designed to solve that task. Specifically, GNNs have been found to align with DP, a paradigm capable of solving many CO problems. This alignment between DP and GNNs has been further examined by \citet{NEURIPS2022_8248b1de} using methods from category theory and abstract algebra. 

However, network interdiction problems fall into a category that doesn't align with DP and thus cannot be naturally addressed by GNNs.  DP relies on the principle of optimal substructure, where an optimal solution to a problem can be constructed from optimal solutions to its subproblems \cite{cormen2022introduction}. Network interdiction problems typically involve nested two-layer decision-making, which does not exhibit the overlapping subproblems characteristic. 
The way our proposed method aligns better with the exact solver can also be explained with \citet{NEURIPS2022_8248b1de}'s theories in abstract algebra: 
by representing the network interdiction instances as multipartite graphs, the MMILP-GNN and the target algorithm (BB) at least share the same finite sets in the polynomial spans.

\section{Experimental Evaluation}

In the experiments, we focus on two specific network interdiction problems: the shortest path interdiction problem and the maximum flow interdiction problem.

\textbf{Instance generation.} 
We generate four maximum flow interdiction problem sets, MFI20, MFI30, MFI100, and MFI200, consisting of 4,000 of 20, 30, 100, and 200-node maximum flow interdiction instances, respectively. Similarly, we generate three shortest-path interdiction problem sets, SPI20, SPI30, and SPI100. 
Further information on data generation is available in the Appendix \ref{appdix:data}.

\textbf{Baseline methods. } 
\begin{itemize}[leftmargin=10pt]
    \item \textbf{GCN}: The Graph Convolutional Network (GCN) is a popular GNN variant \cite{kipf2017semisupervised}.
    \item \textbf{rGIN}: The GIN model is a standard isomorphism graph neural network that operates on the original graph of the instances \cite{xu2018how}. To ensure a fair comparison, we use the version of GIN that incorporates random features, denoted as rGIN \cite{sato2021random}.
    \item \textbf{SCIP} is the most powerful non-commercial optimization software package that incorporates BB \cite{BestuzhevaEtal2021OO}. We formulate both the maximum flow interdiction problem and the shortest path interdiction problem as MILP problems and allow SCIP to solve them. The detailed configuration can be found in Appendix \ref{appendix:scip}. 
\end{itemize}

\textbf{Details on validation and testing}: 
For each of the datasets, we use 50\% of instances for training, 25\% for validation, and 25\% for testing. The MMILP-GNN produces interdiction marginal probabilities for each edge in the graph as its final output. Given the predicted $\hat{x}$ in Eq. (\ref{eq:output_layer}), we utilize two decision-making strategies to assess the model's performance.
\begin{itemize}[leftmargin=10pt]
    \item \textbf{End-to-end prediction}: For a problem with a limited interdiction budget of $k$, the end-to-end prediction strategy simply applies interdictions to the top-k edges based on the predicted marginal interdiction probability. 
    \item \textbf{Predict-and-search}: This strategy utilizes a trust-region-based algorithm to search for high-quality feasible solutions, guided by a mapping derived from the initial predictions \cite{han2023a}.  The implementation of this strategy is directly derived from existing work outlined by \citet{han2022gnn} and is detailed in Appendix \ref{appendix:strategy}. 
\end{itemize}

We employ these two decision-making strategies tailored for different purposes and comparisons. Firstly, it's important to note that the end-to-end prediction method is \textit{not} expected to outperform SCIP,  because SCIP is expected to generate \textit{optimal} solutions, while perfect generalization ability is theoretically unattainable for machine learning models (detailed in Section \ref{sec:e2e_res}). As a result, the end-to-end prediction method is only compared with other learning-based methods. 
Secondly, to facilitate a fair comparison between our learning-based method and the SCIP, we adopt the predict-and-search strategy. The rationale is that although SCIP is capable of providing the optimal solution, it typically requires a significant amount of time. Within a given timeframe, the predict-and-search strategy allows us to refine the predictions from our models. The goal is to assess whether our model can swiftly and reliably predict a smaller region than the original search space, enabling us to obtain better solutions faster than SCIP before it reaches an optimal solution.

\subsection{End-to-End Prediction Results} \label{sec:e2e_res}

For end-to-end predictions, we conduct experiments on three small datasets for both the shortest-path interdiction problem and the maximum flow interdiction problem. Table \ref{table:e2e_mfi} demonstrate that Neural Interdiction consistently outperforms the comparison models. 
The approximation ratio measures the average ratio to the optimal results: the higher, the better for max-min problems, and the lower, the better for min-max problems.
Visualization of the predictions is available in Figure \ref{fig:e2e_fm1}. It's important to acknowledge that learning-based methods, including our proposed Neural Interdiction, are not expected to match SCIP's performance because both interdiction problems are recognized as NP-hard problems, while neural network models operate in polynomial time. Hence, under the assumption that $\text{P} \neq \text{NP}$, it is clear that learning-based methods cannot approximate exact solvers.

\begin{figure}[htp] 
    \centering
    \includegraphics[width=0.6\linewidth]{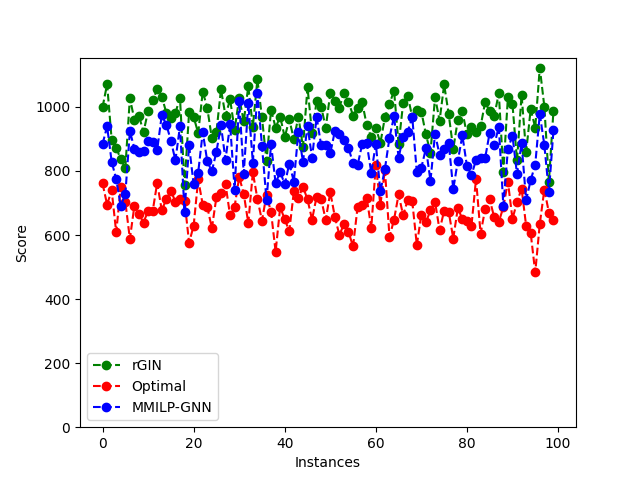}
    \caption{Example Shortest Path Interdiction Instance}
    \label{fig:e2e_fm1}
\end{figure}

\begin{table}[h]
\centering
\begin{tabular}{clll}
\hline
\textbf{Dataset}               & \textbf{Method}              & \textbf{Approximation Ratio } & \textbf{Optimality Gap} \\ \hline
\multirow{3}{*}{MFI20} & rGIN                &     $1.33 \pm 0.07 $           &     $14.2 \pm 0.92$      \\
                      & GCN                 &    $1.42\pm 0.06$       & $15.2 \pm 0.65$              \\
                      &\textbf{ Neural Interdiction} &   $\mathbf{1.22} \pm \mathbf{0.06}$    & $\textbf{8.2} \pm \textbf{0.32}$              \\ \hline
\multirow{3}{*}{MFI30} & rGIN                &   $1.43\pm 0.02$               &  $23.2 \pm 0.53$   \\
                      & GCN                   &   $1.59\pm 0.07$    &       $25.8 \pm 0.97$         \\
                      & \textbf{Neural Interdiction} &   $\mathbf{1.28}\pm \mathbf{0.02}$     &   $\textbf{9.6} \pm \textbf{0.69} $    \\ \hline
\multirow{3}{*}{MFI00} & rGIN                &   $1.48\pm 0.19$    &  $46.2 \pm 3.48$         \\
                      &  GCN                &   $1.59\pm 0.23$    &   $59.2 \pm 5.89$         \\
                      & \textbf{Neural Interdiction} &   $\mathbf{1.33}\pm \mathbf{0.17}$  &    $\textbf{38.2} \pm \textbf{3.83}$     \\ \hline
\multirow{3}{*}{SPI20} & rGIN                &  $0.75\pm0.18$ &   $19.75\pm2.56$   \\
                      & GCN                 &  $0.69\pm0.17$  & $23.25\pm3.37$   \\
                      & \textbf{Neural Interdiction }& $\textbf{0.82}\pm\textbf{0.14}$  &  $\textbf{18.25}\pm\textbf{2.73}$   \\ \hline
\multirow{3}{*}{SPI30} & rGIN                &   $0.68\pm0.25$    &   $20.5\pm2.66$    \\
                      &  GCN                   &   $0.64\pm0.29$   &  $26.25\pm5.25$  \\
                      & \textbf{Neural Interdiction} &   $\textbf{0.81}\pm\textbf{0.17}$  &  $\textbf{19.75}\pm\textbf{2.37}$  \\ \hline
\multirow{3}{*}{SPI00} & rGIN                &  $0.72\pm0.22$     &   $24.96\pm5.24$ \\
                      &  GCN                   &  $0.56\pm0.26$  & $27.75\pm6.38$   \\
                      & \textbf{Neural Interdiction} &   $\textbf{0.75}\pm\textbf{0.19}$   &   $\textbf{23.25}\pm\textbf{4.18}$  \\ \hline
\end{tabular}
\caption{Results of end-to-end predictions. The approximation ratio is the higher, the better for SPI problems, and the lower, the better for MFI problems.}
\label{table:e2e_mfi}
\end{table}
\subsection{Predict-and-Search Results}

Using the predict-and-search strategy, we initially predict interdiction decision distributions with MMILP-GNN, followed by searching for near-optimal solutions within a trust region derived from the prediction. We then assess Neural Interdiction's performance against SCIP within a limited time frame. In the case of Neural Interdiction, the inference time comprises two periods: model inference and search. Model inference, conducted via PyTorch on GPUs, is relatively negligible compared to search time. The search phase, also executed through SCIP, can be contrasted with the solving time of the original problem. 

Due to limited resources for labeling and training, we conducted experiments on 2,000 instances of 200-node maximum flow interdiction problems. Figure \ref{fig:ps_fm1} represents a comparison of solving times between the original problem solved by SCIP and the predict-and-search approach. This figure is based on SCIP logs from running both SCIP alone and SCIP combined with Neural Interdiction. In most cases, the pre-solving time is very similar, but our proposed method begins to show advantages starting from 25-35 seconds.  This illustrates that the MMILP-GNN model, trained on the training set, consistently offers reliable estimates of where the optimal solution may lie, such that the search strategy and start from there and find good solutions quicker than solving the exact problem.

\begin{figure}[htp] 
    \centering
    \includegraphics[width=0.6\linewidth]{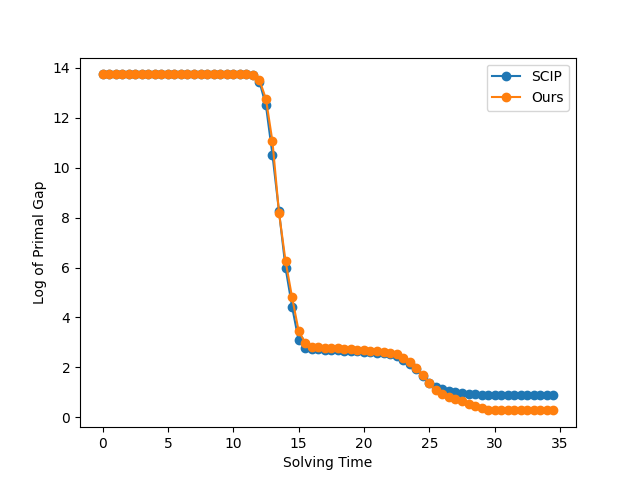}
    \caption{Averaged results visualization from 2,000 runs for each of the two methods.  }
    \label{fig:ps_fm1}
\end{figure}






\section{Conclusions and Limitations}
We introduce a GNN-based framework for solving network interdiction problems by modeling them as single-level MILPs and encoding them onto multipartite graphs. Using a multipartite GNN, we learn the mapping between the graph representation and solution distributions, allowing us to derive final discrete solutions through two strategies. To our knowledge, this is the first learning-based method for addressing network interdiction problems.  We provide theoretical evidence demonstrating the representational power of our model while also elucidating how the method aligns with conventional algorithms capable of solving such problems.
The main limitation of this work is that, although the experimental results show our model can learn from training data and perform well on several datasets, its performance and generalization ability are not yet sufficient to make it a reliable solver for network interdiction problems.

\pagebreak 
\bibliographystyle{plainnat}
\bibliography{main}

\pagebreak 

\appendix 
\section{Proofs for Section \ref{sec:theory}} \label{appendix:proof}

Our approach to measuring the representation power of MMILP-GNN is heavily inspired by existing research on the representation power of GNNs \cite{Xu2020What} and the separation power of GNNs in LP \cite{chen2022representing} and MILP \cite{chen2022representing2}. 
 We also prove Lemma~\ref{lem:sameWL2sameWGNN} and Lemma~\ref{lem:sameGNN2sameWL} in the same way corresponding lemmas and theorems are proved in \cite{chen2022representing}.

\begin{proof}[Proof of Lemma~\ref{lem:sameWL2sameWGNN}]

Define $p+1$ sets (not multisets) that collect different colors:
    \begin{equation*}
        \mathbf{C}_{l-1}^V = \{C_1^{l-1,V},\dots, C_m^{l-1,V},\hat{C}_1^{l-1,V},\dots, \hat{C}_m^{l-1,V}\},
    \end{equation*}
    and
    \begin{equation*}
        \mathbf{C}_{l-1}^{W_k} = \{C_1^{l-1,W_k},\dots, C_n^{l-1,W_k},\hat{C}_1^{l-1,W_k},\dots, \hat{C}_n^{l-1,W_k}\}.
    \end{equation*}

    We aim to prove by induction that for any $l\in\{0,1,\dots,L\}$, the followings hold:
    \begin{itemize}
    	\item[(i)] $C_i^{l,V}=C_{i'}^{l,V}$ implies $h_i^{l,V}=h_{i'}^{l,V}$, for $1\leq i,i'\leq m$;
    	\item[(ii)] $\hat{C}_i^{l,V}=\hat{C}_{i'}^{l,V}$ implies $\hat{h}_i^{l,V}=\hat{h}_{i'}^{l,V}$, for $1\leq i,i'\leq m$;
    	\item[(iii)] $C_i^{l,V}=\hat{C}_{i'}^{l,V}$ implies $h_i^{l,V}=\hat{h}_{i'}^{l,V}$, for $1\leq i,i'\leq m$;
    	\item[(iv)] $C_j^{l,W_k}=C_{j'}^{l,W_k}$ implies $h_j^{l,W_k}=h_{j'}^{l,W_k}$, for $1\leq j,j'\leq n_k$;
    	\item[(v)] $\hat{C}_j^{l,W_k}=\hat{C}_{j'}^{l,W_k}$ implies $\hat{h}_j^{l,W_k}=\hat{h}_{j'}^{l,W_k}$, for $1\leq j,j'\leq n_k$;
    	\item[(vi)] $C_j^{l,W_k}=\hat{C}_{j'}^{l,W_k}$ implies $h_j^{l,W_k}=\hat{h}_{j'}^{l,W_k}$, for $1\leq j,j'\leq n_k$.
    \end{itemize}
    The above claims (i)-(vi) are clearly true for $l=0$ due to the injectivity of $\text{HASH}_{V}^0$ and $\text{HASH}_{W_k}^0$. Now we assume that (i)-(vi) are true for some $l - 1 \in\{0,1,\dots,L-1\}$. Suppose that $C_i^{l,V}=C_{i'}^{l,V}$, i.e.,
    \begin{align*}
        & \text{HASH}_{ V}^l \left(C_i^{l-1,V},\sum_{k=0}^{p}\sum_{j=1}^{n_k} E_{i,j}^{W_k, V}\text{HASH}_{W_k}^{l'}\left(C_j^{l-1, W_k}\right) \right)  \\
        = & \text{HASH}_{V}^l \left(C_{i'}^{l-1,V}, \sum_{k=0}^{p}\sum_{j=1}^{n_k} E_{i',j}^{W_k, V}\text{HASH}_{W_k}^{l'}\left(C_j^{l-1, W_k}\right) \right),
    \end{align*}
    for some $1\leq i,i'\leq m$. It follows from the injectivity of $\text{HASH}_{V}^l$ that
    \begin{equation}\label{equal_colorl-1}
        C_i^{l-1,V}=C_{i'}^{l-1,V},
    \end{equation}
    and
    \begin{align*}
        \sum_{k=0}^{p}\sum_{j=1}^{n_k} E_{i,j}^{W_k, V}\text{HASH}_{W_k}^{l'}\left(C_j^{l-1, W_k}\right) 
        = \sum_{k=0}^{p}\sum_{j=1}^{n_k} E_{i',j}^{W_k, V}\text{HASH}_{W_k}^{l'}\left(C_j^{l-1, W_k}\right).
    \end{align*}
    According to the linearly independent property of $\text{HASH}_{W_k}^{l'}$, the above equation implies that
    \begin{equation}\label{equal_weightsum}
        \sum_{C_j^{l-1, W_k} = C} E_{i,j} = \sum_{C_j^{l-1, W_k} = C} E_{i',j},\quad \forall~C\in \mathbf{C}_{l-1}^{W_k}.
    \end{equation}
    Note that the induction assumption guarantees that $h_j^{l-1,W_k} = h_{j'}^{l-1,W_k}$ as long as $C_j^{l-1,W_k} = C_{j'}^{l-1,W_k}$. So one can assign for each $C\in \mathbf{C}_{l-1}^{W_k}$ some $h(C)\in \mathbb{R}^{d_{l-1}}$ such that $h_j^{l-1,W_k} = h(C)$ as long as $C_j^{l-1,W_k} = C$ for any $1\leq j\leq n_k$. Therefore, it follows from \eqref{equal_weightsum} that
    \begin{align*}
       &  \sum_{j=1}^n E_{i,j}f_l^{W_k} (h_j^{l-1, W_k}) =  \sum_{C\in \mathbf{C}_{l-1}^{W_k}} \sum_{C_j^{l-1, {W_k}} = C} E_{i,j} f_l^{W_k} (h(C)) \\
        & \sum_{j=1}^n E_{i',j}f_l^{W_k} (h_j^{l-1, {W_k}}) =  \sum_{C\in \mathbf{C}_{l-1}^{W_k}} \sum_{C_j^{l-1, {W_k}} = C} E_{i',j} f_l^{W_k} (h(C)) \\
       &  \sum_{j=1}^n E_{i,j}f_l^{W_k} (h_j^{l-1, {W_k}}) = \sum_{j=1}^n E_{i',j}f_l^{W_k} (h_j^{l-1, {W_k}})
    \end{align*}
    
    Note also that \eqref{equal_colorl-1} and the induction assumption lead to $h_i^{l-1,V}=h_{i'}^{l-1,V}$. Then one can conclude that
    \begin{align*}
        h_i^{l,V} = & g_l^V\left( h_i^{l-1,V},\sum_{k=0}^{p} \sum_{j=1}^{n_k} E_{i,j}^{W_k, V} f_l^{W_k}(h_j^{l-1, W_k}) \right) \\ 
        = & g_l^V\left( h_{i'}^{l-1,V}, \sum_{k=0}^{p} \sum_{j=1}^{n_k} E_{i',j}^{W_k, V} f_l^{W_k}(h_j^{l-1, W_k}) \right)  = h_{i'}^{l,V}.
    \end{align*}
    This proves the claim (i) for $l$. The other five claims can be proved using similar arguments.
    
	Therefore, we obtain from $(G,H)\sim(\hat{G},\hat{H})$ that
	\begin{equation*}
	    \left\{\left\{h_1^{L,V}, h_2^{L,V},\dots, h_m^{L,V}\right\}\right\} = \left\{\left\{\hat{h}_1^{L,V}, \hat{h}_2^{L,V},\dots, \hat{h}_m^{L,V}\right\}\right\},
	\end{equation*}
	and that 
	\begin{align*}
	     & \left\{\left\{h_1^{L,W_k}, h_2^{L,W_k},\dots, h_n^{L,W_k}\right\}\right\}  = \left\{\left\{\hat{h}_1^{L,W_k}, \hat{h}_2^{L,W_k},\dots, \hat{h}_n^{L,W_k}\right\}\right\}.
	\end{align*}
	By the definition of the output layer, the above conclusion guarantees that $F(G,H) = \sigma(F(\hat{G},\hat{H}))$ for some $\sigma\in S_n$.
\end{proof}

\begin{proof}[Proof of Lemma~\ref{lem:sameGNN2sameWL}]
	It suffices to prove that, if $(G,H)$ can be distinguished from $(\hat{G},\hat{H})$ by the WL test, then there exists $F\in \mathcal{F}$, such that $F(G,H)\neq F(\hat{G},\hat{H})$. The distinguish-ability of the WL test implies that there exists $L\in\mathbb{N}$ and hash functions, $\text{HASH}_{0,V}$, $\text{HASH}_{0,W}$, $\text{HASH}_{l,V}$, $\text{HASH}_{l,W}$, $\text{HASH}_{l,V}'$, and $\text{HASH}_{l,W}'$, for $l=1,2,\dots, L$, such that
 
	\begin{align}\label{diff_WL_V}
		& \left\{\left\{C_1^{L,V},C_2^{L,V},\dots, C_m^{L,V}\right\}\right\}  
  \neq  \left\{\left\{\hat{C}_1^{L,V},\hat{C}_2^{L,V},\dots, \hat{C}_m^{L,V}\right\}\right\},
	\end{align}
 
    or
    
    \begin{align}\label{diff_WL_W}
    	&\left\{\left\{C_1^{L,W},C_2^{L,W},\dots, C_n^{L,W}\right\}\right\} 
     \neq   \left\{\left\{\hat{C}_1^{L,W},\hat{C}_2^{L,W},\dots, \hat{C}_n^{L,W}\right\}\right\},
    \end{align}
    
    We aim to construct some GNNs such that the followings hold for any $l=0,1,\dots,L$:
    \begin{itemize}
    	\item[(i)] $h_i^{l,V}=h_{i'}^{l,V}$ implies $C_i^{l,V}=C_{i'}^{l,V}$, for $1\leq i,i'\leq m$;
    	\item[(ii)] $\hat{h}_i^{l,V}=\hat{h}_{i'}^{l,V}$ implies $\hat{C}_i^{l,V}=\hat{C}_{i'}^{l,V}$, for $1\leq i,i'\leq m$;
    	\item[(iii)] $h_i^{l,V}=\hat{h}_{i'}^{l,V}$ implies $C_i^{l,V}=\hat{C}_{i'}^{l,V}$, for $1\leq i,i'\leq m$;
    	\item[(iv)] $h_j^{l,W_k}=h_{j'}^{l,W_k}$ implies $C_j^{l,W_k}=C_{j'}^{l,W_k}$, for $1\leq j,j'\leq n$;
    	\item[(v)] $\hat{h}_j^{l,W_k}=\hat{h}_{j'}^{l,W_k}$ implies $\hat{C}_j^{l,W_k}=\hat{C}_{j'}^{l,W_k}$, for $1\leq j,j'\leq n$;
    	\item[(vi)] $h_j^{l,W_k}=\hat{h}_{j'}^{l,W_k}$ implies $C_j^{l,W_k}=\hat{C}_{j'}^{l,W_k}$, for $1\leq j,j'\leq n$.
    \end{itemize}
    It is clear that the above conditions (i)-(vi) hold for $l=0$ as long as we choose $f^{V'}_{\mathrm{in}}$ that is injective on the following $p+1$ sets (not multisets) respectively:  $\{h_1^V,\dots,h_m^V,\hat{h}_1^V,\dots,\hat{h}_m^V\}$ and $\{h_1^{W_k},\dots,h_n^{W_k},\hat{h}_1^{W_k},\dots,\hat{h}_n^{W_k}\}.$~ 
    We then assume that (i)-(vi) hold for some $0\leq l-1<L$, and show that these conditions are also satisfied for $l$ if we choose $f_l^V,f_l^W,g_i^V,g_l^W$ properly. Let us consider the set (not multiset):
    \begin{equation*}
    	\left\{\alpha_1,\alpha_2,\dots,\alpha_s\right\}\subset \mathbb{R}^{d_{l-1}} 
    \end{equation*}
    that collects all different values in $h_1^{l-1,W}, h_2^{l-1,W},\dots, h_n^{l-1,W},\hat{h}_1^{l-1,W}, \hat{h}_2^{l-1,W},\dots, \hat{h}_n^{l-1,W}$. Let $d_{l} \geq s$ and let $e_p^{d_l}  = (0,\dots,0,1,0,\dots,0)$ be the vector in $\mathbb{R}^{d_l}$ with the $p$-th entry being $1$ and all other entries being $0$, for $1\leq p\leq s$. Choose $f_l^W:\mathbb{R}^{d_{l-1}}\rightarrow \mathbb{R}^{d_{l}}$ as a continuous function satisfying $f_l^W(\alpha_p) = e_p^{d_l}$, $p=1,2,\dots,s$, and choose $g_l^V:\mathbb{R}^{d_{l-1}}\times \mathbb{R}^{d_{l}} \rightarrow \mathbb{R}^{d_{l}}$ that is continuous and is injective when restricted on the set (not multiset)
    \begin{multline*}
        \left\{\left( h_i^{l-1,V},\sum_{j=1}^n E_{i,j} f_l^W(h_j^{l-1, W}) \right):1\leq i\leq m\right\} 
        \cup \left\{\left( \hat{h}_i^{l-1,V},\sum_{j=1}^n \hat{E}_{i,j} f_l^W(\hat{h}_j^{l-1, W}) \right):1\leq i\leq m\right\}.
    \end{multline*}
    Noticing that
    \begin{equation*}
        \sum_{j=1}^n E_{i,j} f_l^W(h_j^{l-1,W}) = \sum_{p=1}^s\left( \sum_{h_j^{l-1,W} = \alpha_p} E_{i,j}\right) e_p^{d_l},
    \end{equation*}
    and that $\{e_1^{d_l},e_2^{d_l},\dots,e_s^{d_l}\}$ is linearly independent, one can conclude that $h_i^{l,V} = h_{i'}^{l,V}$ if and only if $h_i^{l-1,V} = h_{i'}^{l-1,V}$ and $\sum_{j=1}^n E_{i,j} f_l^W(h_j^{l-1,W})=\sum_{j=1}^n E_{i',j} f_l^W(h_j^{l-1,W})$, where the second condition is equivalent to 
    \begin{equation*}
        \sum_{h_j^{l-1,W} = \alpha_p} E_{i,j} = \sum_{h_j^{l-1,W} = \alpha_p} E_{i',j},\quad\forall~p \in\{1,2,\dots,s\}.
    \end{equation*}
    This, as well as the condition (iv) for $l-1$, implies that
    \begin{equation*}
        \sum_{j=1}^n E_{i,j}\text{HASH}_{l,W_k}'\left(C_j^{l-1, W}\right) = \sum_{j=1}^n E_{i',j}\text{HASH}_{l,W_k}'\left(C_j^{l-1, W}\right),
    \end{equation*}
    and hence that $C_i^{l,V} = C_{i'}^{l,V}$ by using $h_i^{l-1,V} = h_i^{l-1,V}$ and condition (i) for $l-1$. Therefore, we know that (i) is satisfied for $l$, and one can show (ii) and (iii) for $l$ using similar arguments by taking $d_{l}$ large enough. In addition, $f_l^V$ and $g_l^W$ can also be chosen in a similar way such that (iv)-(vi) are satisfied for $l$.
    
    Combining \eqref{diff_WL_V}, \eqref{diff_WL_W}, and condition (i)-(iv) for $L$, we obtain that
    \begin{equation}\label{diff_GNN_V}
    	\left\{\left\{h_1^{L,V},h_2^{L,V},\dots, h_m^{L,V}\right\}\right\} \neq \left\{\left\{\hat{h}_1^{L,V},\hat{h}_2^{L,V},\dots, \hat{h}_m^{L,V}\right\}\right\},
    \end{equation}
    or
    \begin{equation*}
    	\left\{\left\{h_1^{l,W_k},h_2^{l,W_k},\dots, h_n^{l,W_k}\right\}\right\} \neq \left\{\left\{\hat{h}_1^{l,W_k},\hat{h}_2^{l,W_k},\dots, \hat{h}_n^{l,W_k}\right\}\right\}.
    \end{equation*}
    Without loss of generality, we can assume that \eqref{diff_GNN_V} holds. 
    
    Consider the set (not multiset)
    \begin{equation*}
        \{\beta_1,\beta_2,\dots,\beta_t\}\subset \mathbb{R}^{d_{L}}, 
    \end{equation*}
    that collects all different values in $h_1^{L,V},h_2^{L,V},\dots, h_m^{L,V}, \hat{h}_1^{L,V},\hat{h}_2^{L,V},\dots, \hat{h}_m^{L,V}$. Let $k>1$ be a positive integer that is greater than the maximal multiplicity of an element in the multisets $\{\{h_1^{L,V},h_2^{L,V},\dots, h_m^{L,V}\}\}$ and $\{\{\hat{h}_1^{L,V},\hat{h}_2^{L,V},\dots, \hat{h}_m^{L,V}\}\}$. There exists a continuous function $\varphi: \mathbb{R}^{d_{L}} \rightarrow\mathbb{R}$ such that $\varphi(\beta_q) = k^q$ for $q=1,2,\dots,t$, and due to \eqref{diff_GNN_V} and the fact that the way of writing an integer as $k$-ary expression is unique, it hence holds that
    \begin{equation*}
        \sum_{i=1}^m \varphi(h_i^{L,V}) \neq \sum_{i=1}^m \varphi(\hat{h}_i^{L,V}).
    \end{equation*}
    
    Set the dimension of $(L+1)$-th layer as $1$: $d_{L+1} = 1$, 
    \begin{align*}
        & f_\text{out}\left(\sum_{i=1}^m h_i^{L+1,V},\sum_{j=1}^n h_j^{L+1,W}\right) = \sum_{i=1}^m \varphi(h_i^{L,V})  \neq \sum_{i=1}^m \varphi(\hat{h}_i^{L,V}) = f_\text{out}\left(\sum_{i=1}^m \hat{h}_i^{L+1,V},\sum_{j=1}^n \hat{h}_j^{L+1,W}\right),
    \end{align*}
    which guarantees the existence of $F\in\mathcal{F}_{\text{GNN}}$ that has $L+1$ layers and satisfies $F(G,H)\neq F(\hat{G},\hat{H})$.
\end{proof}

\section{Data Generation} \label{appdix:data}

In this study, without loss of generality, we concentrate on two specific network interdiction problems: the shortest path interdiction problem and the maximum flow interdiction problem. For simplicity, we adopt most of the hyperparameter settings from the project by \citet{githubGitHubLucaEliasSchaeferGurobiModels}. The budget constraint for each of the dataset is set to make the 

For the shortest path network interdiction problem, after preprocessing, we solve the problem defined in Eq (\ref{eq:reduction}). For each instance, we generate a fully connected directed graph and assign cost and capacity to each edge. The source node is set as the first node, and the sink node is set as the $N_v$-th node in the network, where $N_v$ is the total number of nodes. The cost on each edge is sampled uniformly between $1.0$ and $10.0$, and the capacity is sampled uniformly between $20.0$ and $70.0$. The budget constraint for three synthetic datasets (SPI20, SPI30, and SPI 100) are all set to $15$, $15$, and $30$, respectively.

Similarly, for the maximum flow network interdiction problem, after preprocessing, we solve the problem in Eq (\ref{exp:2}). For each instance, we generate a fully connected directed graph and assign cost and capacity to each edge. The source node is set as the first node, and the sink node is set as the $N_v$-th node in the network, where $N_v$ is the total number of nodes. The cost on each edge is sampled uniformly between $1.0$ and $10.0$, and the capacity is sampled uniformly between $10.0$ and $60.0$. The budget constraint for all three synthetic datasets (MFI20, MFI30, and MFI 100) is all set to be $15$.

\begin{align}
\begin{split}
      \min_{\alpha, \beta, \gamma} \quad & \sum_{ (i,j) \in A} u_{i,j} \beta_{i,j}, \\
      \textrm{s.t.} \quad & \alpha_i - \alpha_j + \beta_{i,j} + \gamma_{i,j}  \geq 0, \quad \forall(i,j) \in A, \quad \alpha_t - \alpha_s \geq 1, \\
      & \sum_{(i,j)\in A} r_{i,j}\gamma_{i,j} \leq R, \quad \alpha_i \in \{0,1\}, \quad \forall i \in N, \quad \beta_{i,j}, \gamma_{i,j} \in \{0,1\}, \quad \forall(i,j) \in A. 
\end{split}
\label{exp:2}
\end{align}

\section{The Predict-and-Search Strategy} \label{appendix:strategy}

The predict-and-search strategy, as employed in the work of \citet{han2022gnn}, involved predicting solution distributions and subsequently conducting a search for near-optimal solutions within a trust region derived from the predictions. The algorithm detailing this approach is outlined in Algorithm \ref{alg:search}.

\begin{algorithm}[htbp]
	\caption{Predict-and-search Algorithm}
	\textbf{Parameter}: Size $\left\{k_0,k_1\right\}$, radius of the neighborhood: $\Delta$\\
	\textbf{Input}: Instance $M$, Probability prediction $F_{\theta}\left(M\right)$\\
	\textbf{Output}: Solution $x \in \mathbb{R}^n$
	\begin{algorithmic}[1] 
		\State Sort the components in $F_{\theta}\left(M\right)$ from smallest to largest to obtain sets $I_0$ and $I_1$.
		\For{$d=1:n$}
		\If {$d\in I_0\cup I_1$}
		\State \textbf{create binary variable} $\delta_{d}$
		\If {$d\in I_0$}
		\State \textbf{create constraint}\\ $x_{d} \leq \delta_{d}$
		\Else
		\State \textbf{create constraint}\\ $1-x_{d} \leq \delta_{d}$
		\EndIf
		\EndIf
		\EndFor
		\State \textbf{create constraint} $\underset{d\in I_0\cup I_1}{\sum}\delta_{d} \leq \Delta$
		\State Let $M^{\prime}$ denote the instance $M$ with new constraints and variables
		\State Let $x=SOLVE\left(M^{\prime}\right)$
		\State \textbf{return} $x$
	\end{algorithmic}
	\label{alg:search}
\end{algorithm}

\section{SCIP Parameter Settings} \label{appendix:scip}

The parameter settings for the SCIP solver are shown in Table \ref{table:scip}. 

\begin{table}[htbp]\centering
\begin{tabular}{ll}
\hline
Parameter                     & Value      \\
\hline
randomization/randomseedshift & 0          \\
randomization/lpseed          & 0          \\
randomization/permutationseed & 0          \\
separating/maxrounds          & 0          \\
presolving/maxrestarts        & 0          \\
Heuristics                    & AGGRESSIVE \\
\hline
\end{tabular}
\caption{SCIP solver parameter settings .}
\label{table:scip}
\end{table}

\section{Other Experiment Settings}

The implementation of our MMILP-GNN is modified from the branching-imitation module in the ecole package with Pytorch \cite{prouvost2020ecole}. We use ADAM as the optimizer with a 0.0001 learning rate. All experiments were conducted on a desktop with one Intel 12900K CPU and one Nvidia Tesla V100 GPU, with 64GB RAM.

\end{document}